\documentclass{article}

\usepackage{algorithm,algpseudocode}
\usepackage{amsmath}
\usepackage{amsthm}
\usepackage{amssymb}
\usepackage{enumitem}
\usepackage{float}
\usepackage[pdftex]{graphicx}
\usepackage{subcaption}
\usepackage{wrapfig}
\usepackage[colorinlistoftodos]{todonotes}
\usepackage{epstopdf}

\usepackage[final]{corl_2019}

\newcommand{\R}{\mathbb{R}}

\newcommand{\E}{\mathbb{E}}
\newcommand{\Var}{\mathbb{V}ar}

\DeclareMathOperator\KL{KL}
\DeclareMathOperator\ENT{H}
\DeclareMathOperator\Tr{Tr}

\newcommand{\vtheta}{\mbox{\boldmath $\theta$}}

\newcommand{\vmu}{\mbox{\boldmath $\mu$}}

\newcommand{\vpi}{\mbox{\boldmath $\pi$}}

\newcommand{\vXi}{\mathbf \Xi}

\newcommand{\vSigma}{\mathbf \Sigma}

\newcommand{\vf}{\mathbf f}
\newcommand{\vg}{\mathbf g}

\newcommand{\vs}{\mathbf s}

\newcommand{\vu}{\mathbf u}

\newcommand{\vx}{\mathbf x}
\newcommand{\vI}{\mathbf I}

\newcommand{\vR}{\mathbf R}

\newcommand{\cA}{\mathcal{A}}
\newcommand{\cB}{\mathcal{B}}

\newcommand{\cE}{\mathcal{E}}

\newcommand{\cL}{\mathcal{L}}

\newcommand{\cN}{\mathcal{N}}

\newcommand{\cS}{\mathcal{S}}

\newcommand{\cV}{\mathcal{V}}

\newtheorem{theorem}{Theorem}[subsection]
\newtheorem{lemma}{Lemma}[subsection]
\newtheorem{proposition}{Proposition}[subsection]

\algblockdefx{MRepeat}{EndRepeat}{\textbf{repeat}}{}
\algnotext{EndRepeat}

\title{Deep Value Model Predictive Control}

\author{
$^*$\textbf{Farbod Farshidian}$^{1}$, $^*$\textbf{David Hoeller}$^{1,2}$, \textbf{Marco Hutter}$^{1}$ \\
 $^1$ \textrm{Robotic Systems Lab, ETH Zurich}\\
  $^2$ \textrm{NVIDIA}
   \\
  \texttt{\{farbodf, dhoeller, mahutter\}@ethz.ch} \\
}

\begin{document}

\maketitle

%===============================================================================

\begin{abstract}
%In this paper, we derive an actor-critic algorithm, which we refer to as Deep Value Model Predictive Control, that combines model-based trajectory optimization with value function estimation while taking system uncertainties into account. 
%We prove that minimizing the reverse Kullback-Leibler divergence between the current and the optimal state probability distribution (i.e., $\KL(p^\pi \Vert p^*)$) results in an importance sampling scheme that allows for efficient estimation of the value function. 
%The actor minimizes an upper bound of the cross-entropy to the state trajectory distribution of the optimal policy, which allows for efficient estimation of the global value function. In our experiments with a Ballbot system, we show that our algorithm can work with sparse and binary reward signals to efficiently solve obstacle avoidance and target reaching tasks. Compared to previous work, we show that including the value function in the running cost of the trajectory optimizer speeds up the convergence. We also discuss the necessary strategies to robustify the algorithm in practice.
In this paper, we introduce an actor-critic algorithm called Deep Value Model Predictive Control (DMPC), which combines model-based trajectory optimization with value function estimation. The DMPC actor is a Model Predictive Control (MPC) optimizer with an objective function defined in terms of a value function estimated by the critic. We show that our MPC actor is an importance sampler, which minimizes an upper bound of the cross-entropy to the state distribution of the optimal sampling policy. In our experiments with a Ballbot system, we show that our algorithm can work with sparse and binary reward signals to efficiently solve obstacle avoidance and target reaching tasks. Compared to previous work, we show that including the value function in the running cost of the trajectory optimizer speeds up the convergence. We also discuss the necessary strategies to robustify the algorithm in practice.
\footnotetext{$^*$ \textrm{Both authors contributed equally to this work (alphabetical ordering)}}
\end{abstract}

%===============================================================================
\section{Introduction}

% The problem and the general setup
Learning in environments with sparse reward/cost functions remains a challenging problem for Reinforcement Learning (RL). As the exploration strategy employed plays a vital role in such scenarios, the agent has to find and leverage small sets of informative samples maximally. Often, an agent can be provided with prior knowledge about the environment in the form of an incomplete model, such as the agent's dynamics. The sparsity of the reward and potential non-differentiability rule out the possibility of using Trajectory Optimization (TO). Furthermore, the sparsity of the cost function can be problematic even for RL methods that do not use structured and directed exploration policies, e.g., $\varepsilon$-greedy techniques. Thus, the goal of this work is to combine model-based and sample-based approaches in order to exploit the knowledge of the system dynamics while effectively exploring the environment.

% Power of MPC and its drawback
Model Predictive Control (MPC) as a TO technique has proven to be a powerful tool in many robotic tasks \cite{Alexis2011, Farshidian2017MPC, Koenemann2015}. While this model-based approach truncates the time horizon of the task, it continually shifts the shortened horizon forward and optimizes the state-input trajectory based on new state measurements. The main disadvantage of the MPC approach is its relatively high computational cost. As a consequence, the optimization time horizon is often kept short (e.g., in the order of seconds), which in turn prevents MPC from finding temporally global solutions. Furthermore, MPC heavily relies on the differentiability of the formulation and has a hard time dealing with sparse and non-continuous reward/cost signals. 

% Power of RL and its limitation
On the other hand, RL solves the same problem by exploring and collecting information about the environment and making decisions based on samples. The RL agent has to learn about its environment, including its dynamics from scratch via trial and error. Nevertheless, deep reinforcement learning has displayed remarkable performance in long-horizon tasks with sparse rewards \cite{silver2017AlphaGo}, even in the continuous domain \cite{lillicrap2016RL,hwangbo2019RL}. The main drawbacks of RL are that it still requires enormous amounts of data and suffers from the exploration-exploitation dilemma \cite{sutton1998RL}. %Moreover, it has been observed that some modern techniques exhibit performance similar to random search approaches \cite{recht2019RL}.

% Paper Contribution
In this work, we derive an actor-critic approach where the critic is a value function learner and the actor an MPC strategy. Leveraging the generality of value functions, we propose to extend past work such as \cite{lowrey2018plan}, resulting in an algorithm called Deep Value Model Predictive Control (DMPC). The DMPC algorithm uses an MPC policy to interact with the environment and collect informative samples to update its approximation of the value function. The running cost and the heuristic function (also known as the terminal cost) of MPC are defined in terms of the value function estimated by the critic.

We also provide an in-depth analysis of the bilateral effect of this actor-critic scheme. We show that the MPC actor is an importance sampler that minimizes an upper bound of the cross-entropy to the state trajectory distribution of the optimal sampling policy. Using the value function to define the MPC cost enables us to transform an initially stochastic task into a deterministic optimal control problem. We further empirically validate that defining a running cost in addition to the heuristic function accelerates the convergence of the value function, which makes the DMPC algorithm suitable for tasks with sparse reward function.   
%===============================================================================

\section{Problem Formulation}
\label{sec:problem_formulation}

In this section, we provide more details on the control problem we solve. We consider the problem of sequential decision making in which an agent interacts with an environment to minimize the cumulative cost from the current time and onwards. We formulate this problem as a discounted Markov Decision Processes (MDP)~\cite{sutton1998RL} with a stochastic termination time.  

The agent interacts with an environment with state $\vs(t) \in \cS$ and performs actions $\mathbf{u}_\pi(t) \in \cA$ according to a time and state-dependent control policy $\vpi(t, \vs)$. For the sake of brevity, the dependency of $\vs$ and $\vu$ on $t$ is dropped when clear from the context. We consider stochastic dynamical systems where the state evolves according to
\begin{align}
\label{eq:stochastic_process}
    & d \vs = \Big( \vf(t, \vs) + \vg(t,\vs) \vu_{\pi} \Big) dt + \vg(t,\vs) d\cB(t), \quad \vs(t_0) = \vs_0 \\
    & \vu_{\pi} = \pi(t, \vs)
\end{align}
where $\cB(t)$ is a Brownian motion with $\Var[d\cB(t)] = \vSigma dt$ and $\vs_0$ is the initial state at time $t_0$.

The agent selects actions according to a policy $\vpi$ to minimize the discounted expected return for a set of initial states $\vs_0 \in \cS_0$. The discounted expected return is defined as 
\begin{align}
\label{eq:stochastic_cost}
    & V^{\pi}(t_0, \vs_0) =  \mathbb{E}_{\pi} \big[ C^{\pi}(t_0,\vs_0) \big] ,
    \\
\label{eq:path_cost}
    & C^{\pi}(t_0, \vs_0) =  \gamma^{T-t_0} q_f(\vs(T)) + \int_{t_0}^{T} 
    \Big( \gamma^{t-t_0} q(t,\vs) dt + \frac 12 \vu_{\pi}^\top \vR \, \vu_{\pi}    
     dt + \vu_{\pi}^\top \vR d\cB \, \Big),
\end{align}
where $T \in [0,+\infty)$ is the termination time of the problem, $\gamma \in [0,1)$ the discount factor, ${ q(t,\vs): \R^+ \times \cS \to \R }$ the running cost, $q_f(\vs): \cS \to \R$ the termination cost, $\vR=\vR^\top$ a positive definite matrix regularizing the control inputs. $C^{\pi}(t_0,\vs_0)$ is a stochastic variable called path cost. $V^{\pi}(t_0,\vs_0)$ is computed by averaging over path costs generated from rollouts of the stochastic process in Equation~\eqref{eq:stochastic_process} given a policy $\vpi$.

\section{Preliminaries}
\label{sec:preliminaries}

\paragraph{Path Integral Optimal Control}
\label{sec:path_integral}

The optimal policy $\vpi^*(t, \vs)$ and the value function $V^*(t, \vs)$ it induces can be computed according to
\begin{align}
\label{eq:optimal_policy}
    \vpi^*(t, \vs) =& \arg\min_{\pi} V^{\pi}(t, \vs) ,
    \\
\label{eq:optimal_value}
    V^*(t, \vs) =& \min_{\pi} V^{\pi}(t, \vs). 
\end{align}

The optimal value function, $V^*$, satisfies the stochastic Hamilton-Jacobi-Bellman (HJB) equation. 
The derivation of the HJB equation is based on the principles of dynamic programming. This formulation is quite general, but unfortunately, computing an analytical solution is only possible for some special cases such as LQ regulators. 
The original work of Kappen~\cite{kappen2007, kappen16adaptive} studies one of these cases in which the non-linear HJB equation can be transformed into a linear equation by enforcing a constraint on the input regularization matrix and the covariance of the noise:  $\vSigma \, \vR = \lambda \vI$ where $\lambda \in \R^+$. As a result of this linearity, the backward computation of the HJB equation can be replaced by a forward diffusion process that can be computed by stochastic integration. Therefore, the stochastic optimal control solution can be estimated with a Monte Carlo sampling method, resulting in the path integral control formulation
\begin{align}
% \label{eq:optimal_policy}
%     & \vpi^*(t, \vs) = \pi(t, \vs) 
%         + \lim_{dt \to 0^+} { \frac{ \E_{\pi} [\cB(t+dt) e^{- \frac{1}{\lambda} C^{\pi}(t,\vs)}] } {\Psi (t, \vs)} }
%     \\
\label{eq:optimal_value_solution}
    & V^*(t, \vs) = -\lambda \log \Psi^* (t, \vs)
    \\
\label{eq:optimal_distribution}
    & p^* (\rho) = \frac{1}{\Psi^* (t, \vs)} p^\pi (\rho) e^{- \frac{1}{\lambda} C^{\pi}(t,\vs)}
    \\
\label{eq:desirability}
    & \Psi^* (t, \vs) = \E_\pi [e^{- \frac{1}{\lambda} C^{\pi}(t,\vs)}],
\end{align}
where $\rho$ is a state trajectory in the time interval $[t,T]$, $p^\pi(\rho)$ its corresponding probability distribution under policy $\vpi$, and $\Psi^*(t, \vs)$ is called the desirability function for the optimal policy. 
%Note that in Equation~\eqref{eq:desirability}, the desirability function is independent of the policy used to generate rollouts. \farbod{It is actually a function of $\vu$, therefore a function of policy}

Equations \eqref{eq:optimal_value_solution} and \eqref{eq:optimal_distribution} give an explicit expression for the optimal value function and the optimal distribution of state trajectories in terms of the expectation of the cumulative cost over trajectories. The quality of this estimation depends on how well we can estimate the desirability function. In general, for problems with sparse cost functions, using the path integral approach is challenging and requires the use of more advanced approaches to extract samples.   

\paragraph{Cross Entropy} 

Efficient estimation of the solution to the path integral control problem critically depends on the quality of the collected samples. 
As shown by \cite{thijssen2015}, the best sampling strategy for estimating the optimal solution is, ironically, the optimal policy itself\footnote{Thijssen and Kappen \cite{thijssen2015} show that the variance of the desirability function estimation approaches to zero when the optimal policy is used as the sampling distribution ($\pi=\pi^*$)}. 
Based on this observation, open-loop approaches such as \cite{theodorou2010generalized, williams2016aggressive} use the latest estimate of the optimal policy to generate trajectory samples. However, they do not provide a systematic approach to control the variance of the sampling policy, which could lead to an inefficient sampling method, in particular, if the underlying system is unstable.

Using a policy that also controls the variance of the sampled state trajectories is required to have a state-dependent feedback \cite{kappen16adaptive}. However, finding such a feedback policy is equivalent to estimating the optimal sampling policy over the entire state space for each time. The challenge of this approach lies in the design and the update scheme of such a policy.

A common approach to tackle this issue is the cross-entropy method, which is an adaptive importance sampling scheme to estimate the probability of rare events. In this approach, the optimal sampler is estimated by a sequence of more tractable distributions, which are iteratively improved based on the collected samples. For that matter, the cross-entropy between the state probability distributions of the optimal ($p^*$) and the current sampler ($p$) is minimized, where the cross-entropy is defined as 
\begin{align*}
    \ENT(p^*, p) = \ENT(p^*) + \KL(p^* \Vert p).
\end{align*}
It follows that minimizing the cross-entropy with respect to $p$, is equivalent to minimizing the KL divergence between $p^*$ and $p$. This method has been proven to be useful for path-integral based problems. For example, Kappen and Ruiz \cite{kappen16adaptive} have employed a parameterized distribution to estimate the optimal sampler. Similarly, we here use a cross-entropy approach to estimate the optimal sampler. However, instead of a parameterized distribution family, we use an MPC strategy to estimate the optimal sampler implicitly.

\paragraph{Model Predictive Control} 

The MPC strategy replaces the infinite-horizon optimization problem by a sequence of finite-horizon optimal control problems with prediction horizon $H$, which are numerically more tractable. At each control step, MPC solves the following optimal control problem from the current state and time. Then, only the first segment of the optimized sequence is used until the controller receives a new state and repeats the procedure.  
\begin{equation}
\label{eq:mpc}
\begin{aligned}
& \underset{\pi}{\text{minimize}}
& & \gamma^{H} \phi(\vx (t+H)) + \int_{t}^{t+H}  \Big(
        \gamma^{\tau-t} l(\tau,\vx) + \frac 12 \vu_{\pi}(\tau)^\top \vR \, \vu_{\pi}(\tau) 
    \Big) d\tau  \\
& \text{subject to}
& & d \vx(\tau) = \Big( \vf(\tau, \vx) + \vg(\tau,\vx) \vu_{\pi}(\tau) \Big) d\tau, \quad \vx(t) = \vs \\
& & & \vu_{\pi}(\tau) = \pi(\tau, \vx(\tau)).
\end{aligned}
\end{equation}
Here, $l(\tau,\vx): \R^+ \times \cS \to \R $ is the running cost, $\phi(\vx): \cS \to \R$ the heuristic function which accounts for the truncated-time accumulated cost. We denote the state by $\vx$ to clearly distinguish between the actor's computation and the state of the MDP (indicated by $\vs$). Note that the system dynamics are deterministic, and therefore the optimal control problem is formulated deterministically.
%===============================================================================
\section{Deep Value Model Predictive Control}
\label{sec:dmpc}

Our DMPC algorithm is an actor-critic approach where a value function is used to provide a measure of how well an MPC actor performs. We assume that a model of the robot dynamics is available to an agent, and this internal nominal model is accurate. In the following, we briefly discuss the structure of the critic and the actor.
%
% In this section, we derive the DMPC algorithm. First, we show that the forward KL divergence between the optimal and the current state trajectory distribution is upper bounded by its corresponding reverse KL divergence. Then, assuming we know the optimal value function, we formulate a cost functional for the trajectory optimizer that minimizes that reverse KL divergence. As described in section \ref{sec:cross_entropy}, this results in an importance sampling scheme to collect trajectories from the optimal sampler. Using an actor-critic scheme, we can therefore continually update our estimate of the optimal value function.
\noindent
\begin{figure}[!t]
\begin{minipage}{0.5\textwidth}
\centering
\begin{algorithm}[H]
    \caption{Deep Value MPC}\label{alg:dmpc}
    \begin{algorithmic}[1]
        \State \textbf{given} An initial estimate $\hat{V}$ of $V^*$
        \MRepeat
    	\State Sample an initial state $\vx$
    	\State $t := 0$
        \While{$t < T $} 
          \State Compute $\vu_{\hat{\pi}}(t)$ according to Eq.~\eqref{eq:equivalent_deterministic} 
          \State Execute $\vu_{\hat{\pi}}(t)$, obtain a cost $c(t)$
          \State Add the transition tuple to a buffer
          \State $t = t + \delta t $
        \EndWhile
        \State Using the buffer, update $\hat{V}$ (Eq.~\eqref{eq:bellman})
        \EndRepeat
    \end{algorithmic}
\end{algorithm}
\end{minipage}
\hspace{0.5cm}
\begin{minipage}{0.45\textwidth}
\centering
\label{fig:framework}
\includegraphics[width=\textwidth]{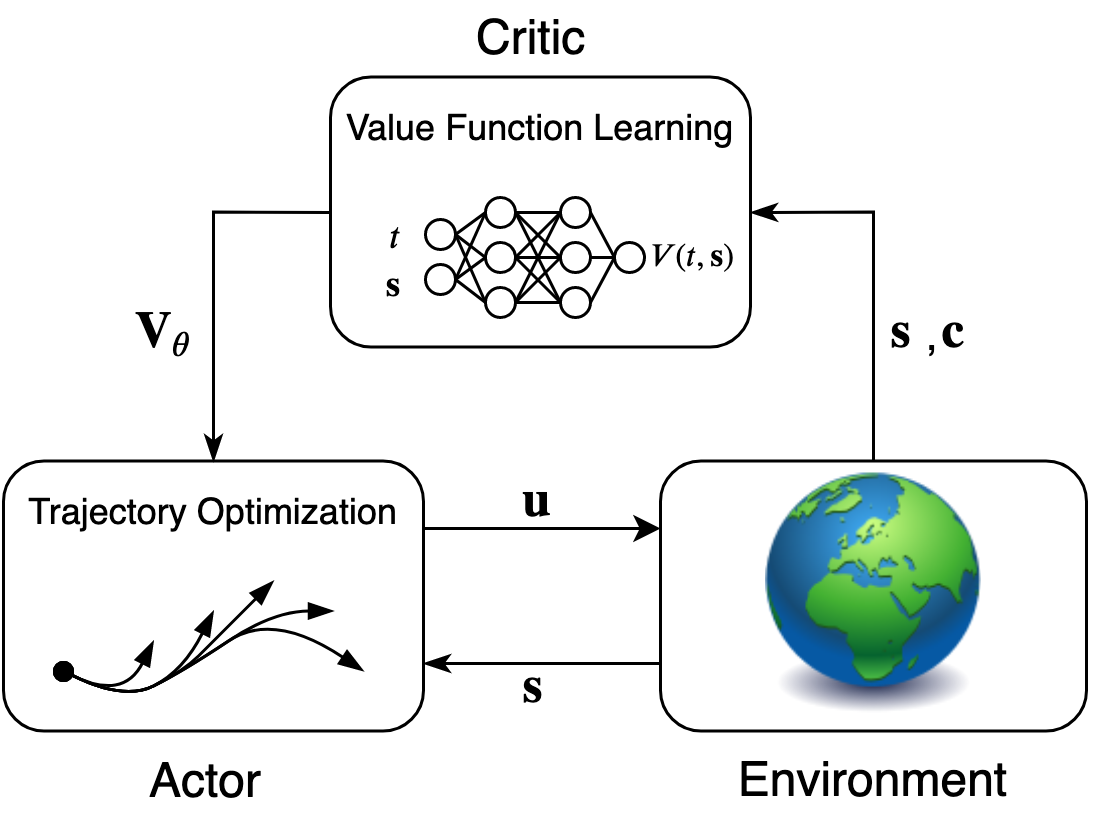}
\captionof{figure}{The DMPC pipeline. The actor is a trajectory optimizer that depends on the value learned by the critic.} 
\end{minipage}%
\end{figure}%

\paragraph{DMPC Critic}
The goal of the critic is to asses the performance of the actor by means of the value function. When the actor interacts with the environment (i.e. rolls out the policy), it collects transition tuples \cite{sutton1998RL}. 
Using these samples, the critic is able to compute the empirical value along each trajectory and thus refine its estimate of the actual value function. The value function is represented by a function approximator parametrized by $\vtheta$. Computing the one step Bellman target for a sample taken at time $t$, i.e.,
\begin{align}
y(t) = c(t) + \gamma \hat{V}(t + \delta t, \vs(t + \delta t) | \boldsymbol{\theta})),
\end{align}
where $\delta t$ is the timestep and $c(t)$ a cost reflecting the performance on a given task, the critic can refine the value function estimate $\hat{V}$ by solving 
\begin{align} \label{eq:bellman}
\underset{\vtheta}{\text{minimize}} \thinspace \E \left[ \left( y(t) - \hat{V}(t, \vs(t) | \boldsymbol{\theta}) \right) ^2 \right] .
\end{align}
Similar to \cite{lillicrap2016RL}, while the target $y$ depends on $\boldsymbol{\theta}$, we neglect this dependency during the optimization.

\paragraph{DMPC Actor}
The DMPC actor is a trajectory optimizer as defined in equation~\eqref{eq:mpc}, where the heuristic function and the running cost are defined as
\begin{align}
    &\phi(\vx(t+H)) = \hat{V}(t+H,\vx(t+H)) ,
    \\
    &l(\tau, \vx) = 
        - \partial_t \hat{V}(\tau,\vx) 
        - \vf(\tau,\vx)^\top \partial_{\vx} \hat{V}(\tau,\vx)
        + \frac 12 \partial_{\vx} \hat{V}(\tau,\vx)^\top \mathbf{\Xi}(\tau, \vx) \, \partial_{\vx} \hat{V}(\tau,\vx) ,
\end{align}
with $\mathbf{\Xi}(\tau, \vx) = \vg(\tau,\vx) \vR^{-1} \, \vg(\tau,\vx)^\top$.

This MPC formulation replaces the original $T$-horizon problem by a sequence of finite-horizon optimal control sub-problems with prediction horizon $H$ ($H < T$). The dynamic programming principle ensures that if the sub-problems are formulated with the exact value function as a heuristic function, then the MPC method solves the problem globally. However, in practice, the actor only has an approximation of the optimal value function. 
While the approximation error degrades the performance of greedy policies, the MPC benefits from the look-ahead mechanism. It can be shown that if the horizon of the actor is long enough, the effect of the value function error is mitigated \cite{lowrey2018plan}. Therefore, MPC strategies using an approximate value function for the heuristic function, in general, outperform methods that are based on the instantaneous minimization of the value function estimate. 

Another significant advantage of using such an approach is that we can compute a temporally global optimal sequence of actions using the temporally local solution of the MPC. This further allows us to tune the MPC time horizon based on the available computational resource and use the value function to account for the truncated accumulated cost.  

\paragraph{Actor-Critic Interaction}
The procedure for the DMPC algorithm directly follows: From an estimate $\hat{V}$ of $V^*$ given by the critic, the MPC actor computes the policy $\vpi$ by solving Equation~\eqref{eq:mpc}. The critic improves the estimate $\hat{V}$ based on the collected samples using Equation~\eqref{eq:bellman}. This results in an off-policy actor-critic method since instead of assessing the value of the current policy, the actor directly estimates the optimal value function. The main steps of DMPC are outlined in Algorithm~\ref{alg:dmpc}. 

The MPC policy is an importance sampler for the value function learner. To motivate this, consider a problem with a sparse cost. Since the MPC actor predicts the state evolution, it can foresee less costly areas of the state space and propagate this information back to the current time. As a result, it can coordinate the action sequence to steer the agent towards future rewarding regions. In the next section, we provide more formal insights and show that MPC minimizes an upper bound of the cross-entropy between the state trajectory distribution of the optimal policy and the current policy. 

By repeating the actor-critic interactions, the estimation of the value function gets closer to the optimal one, which, in turn, improves the MPC performance by enhancing the estimate of the truncated cumulative cost.

\section{DMPC Properties}
In this section, we analyze the properties of the DMPC algorithm. We start with a setup where the MPC actor uses the learned value function only for the heuristic function. There, we study the effect of this actor-critic setup on the convergence of the value function learning. Next, we propose a setting where a running cost is defined based on the value function. We discuss that such a cost extends the application of the actor-critic method to problems with a sparse and non-differentiable cost. Moreover, it allows us to use a deterministic MPC solver instead of a stochastic one.

\subsection{Role of the Heuristic Function}
As a first step, we focus on the case where we only use the learned value function as the heuristic function of the MPC actor. This setup is similar to the approach proposed by~\cite{lowrey2018plan}. We here build upon their result and provide a more in-depth analysis of the impact of the MPC actor on the acceleration of the value function convergence. 

As discussed, the convergence of the value function critically depends on the sampling distribution and, the optimal sampler for the stochastic optimal control problem defined in \eqref{eq:stochastic_process}-\eqref{eq:path_cost} is the optimal control policy $\pi^*$ \cite{thijssen2015}. However, we cannot compute the optimal distribution during the learning process, and instead, we wish to find the near-optimal control policy $\pi$ such that $p^{\pi}$ is close to $p^*$. As discussed in the introduction of the cross-entropy method, we wish to minimize 
\begin{equation}
    \KL(p^* \Vert p^\pi) = -\E_{p^{*}} \left[ \log \left( \frac{p^{\pi}}{p^{*}} \right) \right].
\end{equation}
\begin{theorem} 
\label{thm:kl_upper_bound}
Assuming that from a state $\vs(t)$ the policy remains in the vicinity of the optimal policy up to time $T$, i.e., $\int_t^T \frac 12 \lVert \vu_*(\tau)-\vu_\pi(\tau) \rVert_{\vR}^2 d\tau \leq \cE $, an upper bound for the forward KL divergence between the state trajectory distributions of the optimal policy and policy $\pi$ is given by
% the optimal state trajectory distribution and the state trajectory distribution under policy $\pi$ is given by
%
\begin{align}
    \KL(p^* \Vert p^\pi) 
    \leq 
    \KL(p^\pi \Vert p^*) 
    + \Big(
    \cE \,
    \KL(p^\pi \Vert p^*)  \,
    \Var \Big[ \frac{e^{- \frac{1}{\lambda} C^{\pi}(t,\vs)}}{\Psi (t, \vs)} \Big] 
    \Big)^{\frac 12}.
\end{align}
\end{theorem}
\begin{proof}
    The proof is provided in Appendix \ref{app:proof:kl_upper_bound}. 
\end{proof}

This shows that $\KL(p^* \Vert p^\pi)$ has an upper bound defined by the reverse KL divergence $\KL(p^\pi \Vert p^*)$ and the variance of the normalized desirability function using the policy $\pi$ for sampling. Next, we show that the MPC actor minimizes this reverse KL divergence and that its performance is bounded by the approximation error of the value function. Moreover, we show that the quality of the estimation improves as the MPC horizon $H$ increases. 

\begin{theorem} 
\label{thm:mpc_upper_bound}
Suppose that $\hat{V}$ is an approximation of the optimal value function with infinity norm ${\cL := \max_{t,s} \vert \hat V(t,\vs) - V^*(t,\vs) \vert}$. The policy $\vpi$ is the solution to the problem \eqref{eq:mpc} with terminal cost $\phi(\vx_f) = \hat V(t_f,\vx_f)$. Then for all MPDs, the reverse KL divergence $\KL(p^\pi \Vert p^*)$ can be bounded as
\begin{align}
    \KL(p^{\pi} \Vert p^*)
    \leq 
    \frac{2 \cL \gamma^H} {\lambda (1-\gamma^H)}
\end{align}
\end{theorem}
\begin{proof}
    The proof is provided in Appendix \ref{app:proof:mpc_upper_bound}. 
\end{proof}
\vspace*{-4mm}
Note that if $H_1<H_2$ then ${ \frac{\gamma^{H_2}} {(1-\gamma^{H_2})} < \frac{\gamma^{H_1}} {(1-\gamma^{H_1})} }$. Therefore, as the horizon increases the MPC is less susceptible to the value function approximation error. 

\subsection{Role of the Running Cost}

In this section, we motivate the choice of the DMPC running cost. However, we do not provide any formal proof. The goal is to see how we can extend the idea of using the value function in the heuristic function to the running cost. We show that, at least for the case where we have the optimal value function, this is indeed possible. Thus, for the following analysis, we assume that the exact optimal value function is provided. However, later in the result section, we will empirically show that for the scenarios with sparse costs, the running cost based on the approximate value function also accelerates the convergence of the critic.

%The idea of deriving the running cost from the value function is similar to the inverse optimal control problem \cite{deng1997stochastic}. The goal is to find a running cost that induces a value function that is similar to the optimal one. Using the HJB equation, we can easily show the following proposition. 

\begin{proposition} \label{thm:running_cost}
The control policy $\pi(t,\vs)$ that minimizes the reverse KL divergence $\KL(p^\pi \Vert p^*)$ of the problem defined in Equations \eqref{eq:stochastic_process}-\eqref{eq:path_cost} is also the solution to the deterministic problem 
\begin{align}
    \vpi(t, \vs) = \arg\min_{\pi} \left\{ 
    \phi_d(\vx(t+H)) + \int_t^{t+H} 
    \Big(
        l_d(\tau, \vx)
        + \frac{1}{2} \vu_\pi^\top(\tau) \vR \, \vu_\pi(\tau)    
    \Big) d\tau 
    \right\}
\label{eq:equivalent_deterministic}
\end{align}
where $l_d(\tau, \vx)$, the running cost, and $\phi_d(\vx(t+H))$, the heuristic function are defined as
\begin{align}
\label{eq:intermediate_cost}
    &l_d(\tau, \vx) = 
        - \partial_t V^*(\tau,\vx) 
        - \vf(\tau,\vx)^\top \partial_{\vx} V^*(\tau,\vx)
        + \frac 12 \partial_{\vx} V^*(\tau,\vx)^\top \mathbf{\Xi}(\tau, \vx) \, \partial_{\vx} V^*(\tau,\vx),
    \\&\phi_d(\vx(t+H)) = V^*(t+H,\vx(t+H)),
\label{eq:terminal_cost}
\end{align}
with $V^*(\tau,\vx)$ is the optimal value function of the stochastic problem defined in \eqref{eq:optimal_value} and the state trajectory evolves based on the following deterministic dynamics 
\begin{align}
\label{eq:deterministic_dynamic}
\dot\vx = \vf(\tau, \vx) + \vg(\tau,\vx) \vu_{\pi}, \quad \vx(t) = \vs
\end{align}
\end{proposition}%
\begin{proof}
    The proof is provided in Appendix \ref{app:proof:running_cost}. 
\end{proof}
\vspace*{-4mm}
Proposition~\ref{thm:running_cost} has an important implication; the primary stochastic optimization problem is transformed into a deterministic one. Therefore, a deterministic MPC solver such as the one defined in Equation~\eqref{eq:mpc} can be used. This ultimately means that in order to find the optimal sampling policy, we only need to solve a deterministic MPC problem. This further allows us to employ more sophisticated tools from deterministic optimal control, e.g., Differential Dynamic Programming (DDP) \cite{mayne1966DDP}. 
%===============================================================================
\section{Results}
\label{sec:results}

In this section, we describe the implementation details of the DMPC pipeline. We describe how the different components are designed and highlight the techniques used to make the interaction between the actor and the critic more robust. We then present the experiment results.% on a ball-balancing robot, Ballbot. 

\paragraph{Critic Implementation}
During the rollout of the policy, the critic stores the transition tuples in a replay buffer. Regularly, it samples $K$ mini-batches of size $N$ to update the value function. The value function is represented by a multilayer perceptron (in our experiments, 3 layers with 12 units each and tanh activations). Following an approach similar to \cite{pong2018TDM} and \cite{schaul2015UVF}, we condition the value function on a goal and a time to reach the goal. It can be interpreted as the value of being in a particular state if a specific target has to be reached within a given amount of time. %Despite increasing the dimensionality of the problem, due to the off-policy nature of DMPC, we can apply data augmentation techniques such as Hindsight Experience Replay \cite{andrychowicz2017HER} to improve the speed of convergence. 
Since the derivatives of the network are computed on the actor side, particular attention has to be given to the architecture and the training procedure. The choice of a differentiable activation function, such as tanh, is necessary to guarantee the differentiability of the whole network. Moreover, we noticed that decaying the weights results in smoother behavior. Finally, the critic uses a target network \cite{lillicrap2016RL} so that the actor only receives a Polyak averaged version of the value function.  

\paragraph{Actor Implementation}
The MPC actor uses a DDP-based algorithm known as SLQ \cite{Farshidian2017SLQ, Farshidian2017MPC}. The Jacobian and the Hessian of the value function network, which are necessary to compute the derivatives of the cost function, are computed using an automatic differentiation library.   

\subsection{Experimental Results}
With the following experiments in simulation, we would like to confirm the theory and intuition derived above. More specifically, we would like to show that the algorithm is capable of handling a sparse binary reward set-up, that using the running cost \eqref{eq:intermediate_cost} yields better convergence, and highlight the importance of the MPC time horizon during learning. 

Our experiments are based on the Ballbot robot (see Appendix \ref{app:experiment_details} for more details). In order to answer the questions above, we design a task where the agent has to reach a target in 3 seconds from any initial position. Additionally, we add walls to the environment so that most of the time, it cannot reach the target via the shortest path, see Figure \ref{fig:result_trajectories}. While the target reaching cost is simply encoded as the euclidean distance from the current position to the goal, the walls are encoded by a termination of the episode with a fixed penalty.
First, we analyze the system for different MPC time horizons. Then, we assess the performance when the running cost \eqref{eq:intermediate_cost} is left out (i.e., we only use the value function as the heuristic function), which corresponds to the vanilla case \cite{lowrey2018plan}. Trajectories of the Ballbot's center of mass for different starting positions are shown in Figure \ref{fig:result_trajectories}.

\footnotetext{\textrm{Supplementary video: \url{https://youtu.be/9p4qHBUZDMA}}}
%The intermediate cost in the trajectory optimizer also includes a term that penalizes deviations from the nominal states except for the position, resulting in an initially stable system. This measure greatly increases the sample efficiency because the network does not have to learn how to stabilize the system before starting to solve the task.
% \begin{figure}[ht]
%     \centering
%     \includegraphics[width=.45\textwidth]{trajectory_new.eps}
%     \includegraphics[width=.45\textwidth]{plot_perf.eps}
%     \caption{(Left) Trajectories of the Ballbot's center of mass for different starting positions denoted in red; the green dot represents the target point. (Right) \david{@david: TODO FIX curves to have equal length and equal font on both plots} Performance on the target reaching task for different time horizons and different cost function set-ups. The performance is evaluated as the euclidean distance to the target at the end of each trajectory.}
%     \label{fig:result_plots}
% \end{figure}

\begin{figure*}[t!]
    \centering
    \begin{subfigure}[t]{0.4\textwidth}
        \centering
        \includegraphics[width=\textwidth]{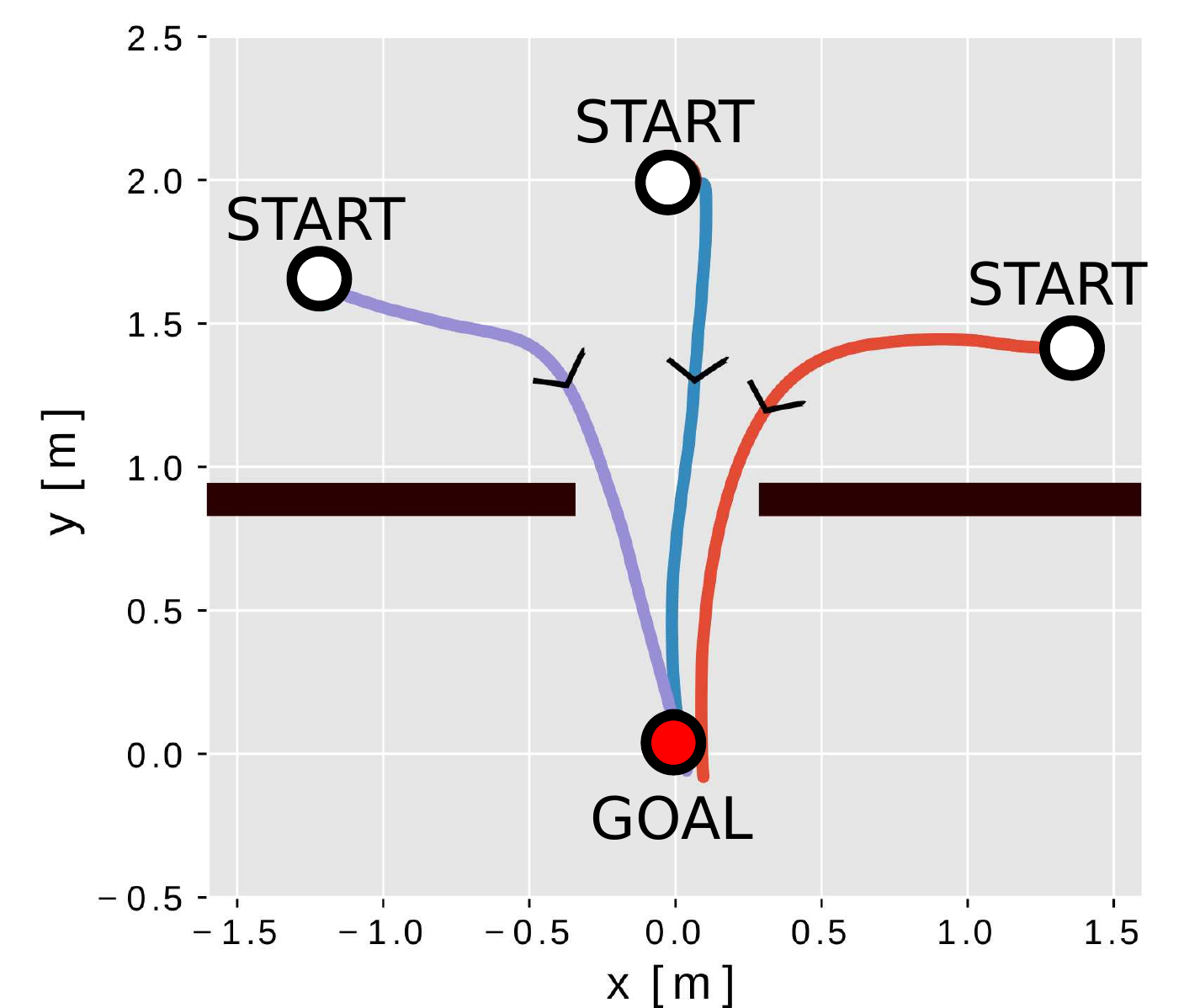}
        \caption{}
        \label{fig:result_trajectories}
    \end{subfigure}%
    ~ 
    \begin{subfigure}[t]{0.45\textwidth}
        \centering
        \includegraphics[width=\textwidth]{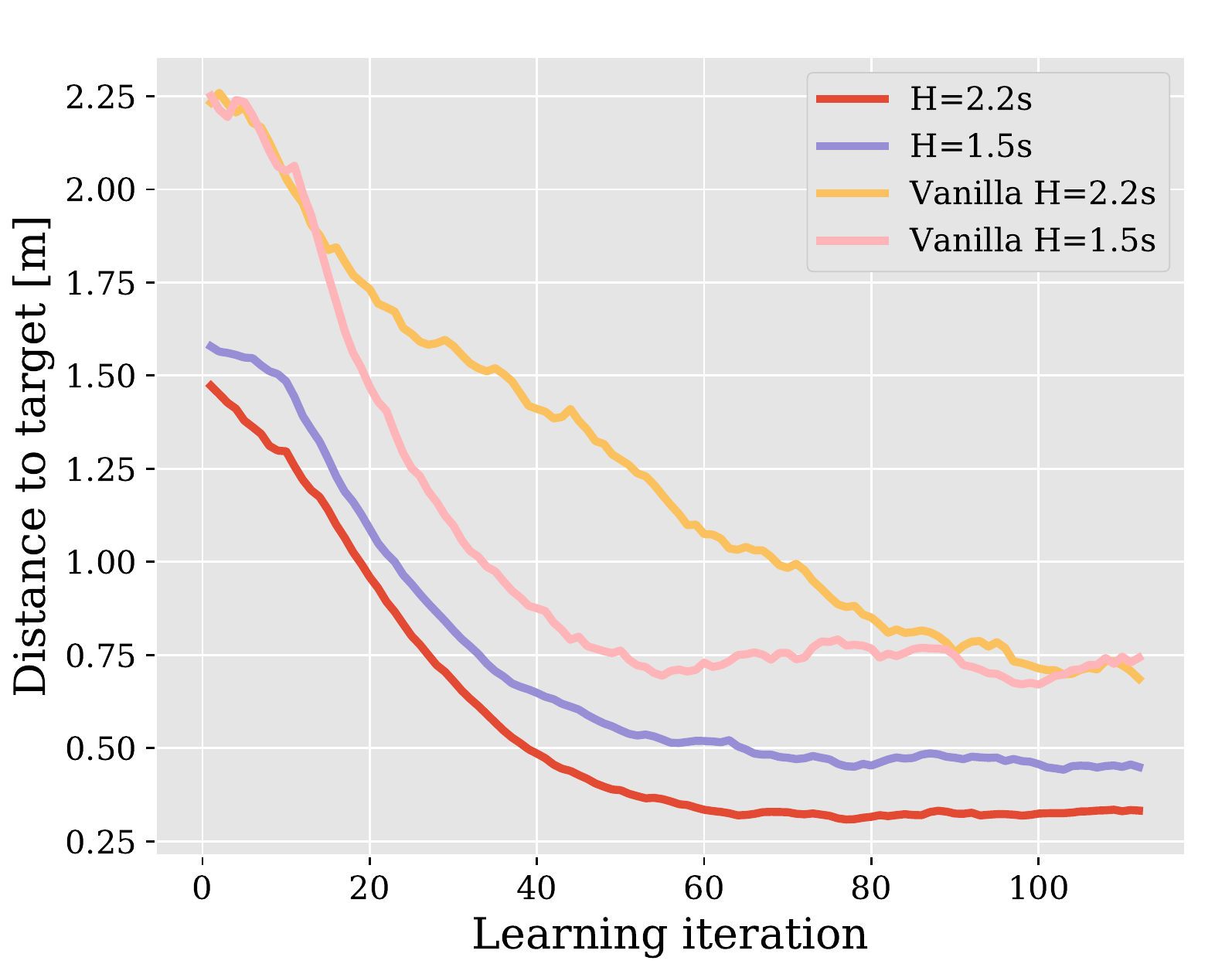}
        \caption{}
        \label{fig:result_perf}
    \end{subfigure}
    \vspace{-2mm}
    \caption{(a) Top down view of the trajectories of the Ballbot's center of mass for different starting positions. (b) Performance on the target reaching task for different time horizons and cost functions.}
    \vspace*{-4mm}
\end{figure*}

% \begin{figure*}[t!]
%   \begin{minipage}[c]{0.4\textwidth}
%     \includegraphics[width=\textwidth]{trajectory_new.eps}
%   \end{minipage}\hfill
%   \begin{minipage}[c]{0.4\textwidth}
%     \includegraphics[width=\textwidth]{plot_perf_2.eps}
%   \end{minipage}\hfill
%   \begin{minipage}[c]{0.2\textwidth}
%     \caption{(a) Trajectories of the Ballbot's center of mass for different starting positions denoted in red, the green dot represents the target point. (b) Performance on the target reaching task for different time horizons and different cost function set-ups. The performance is evaluated as the euclidean distance to the target at the end of each trajectory.
%     } \label{fig:03-03}
%   \end{minipage}
% \end{figure*}

\begin{figure}[htp]
    \centering
    \includegraphics[width=.15\textwidth]{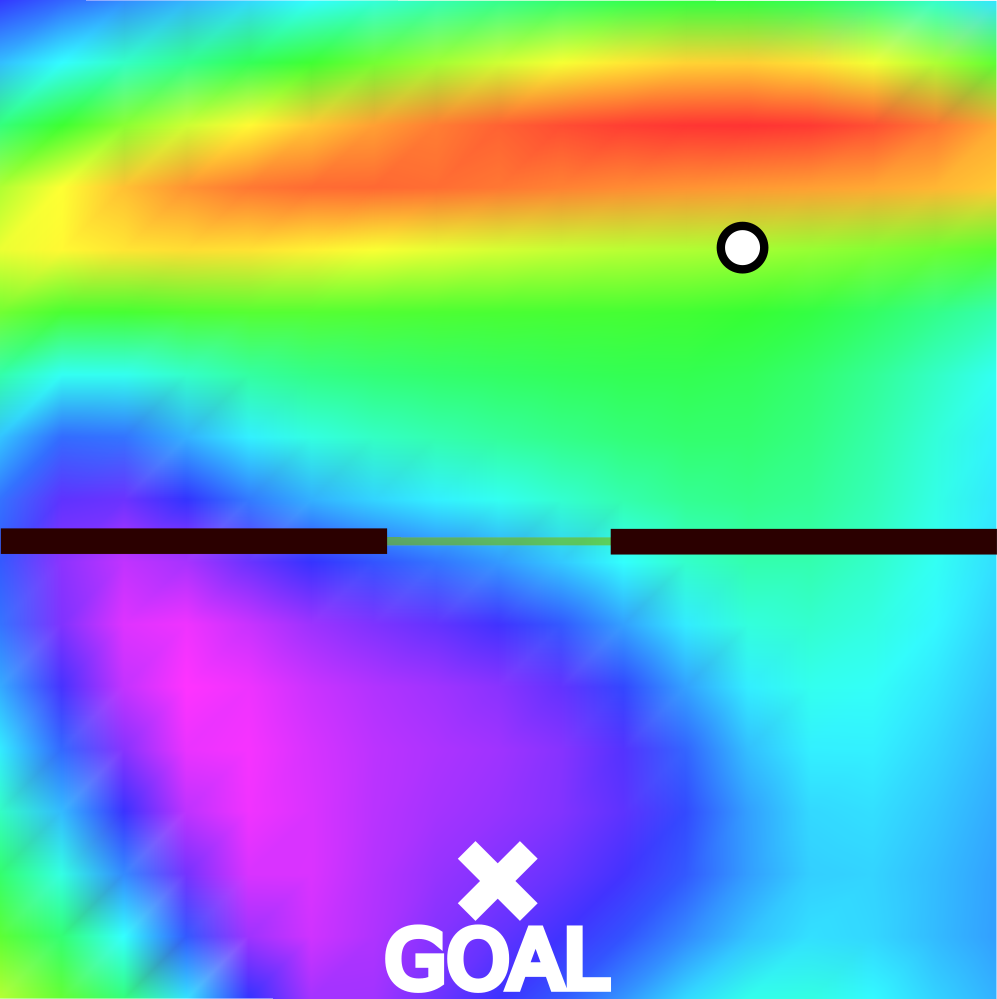}
    \includegraphics[width=.15\textwidth]{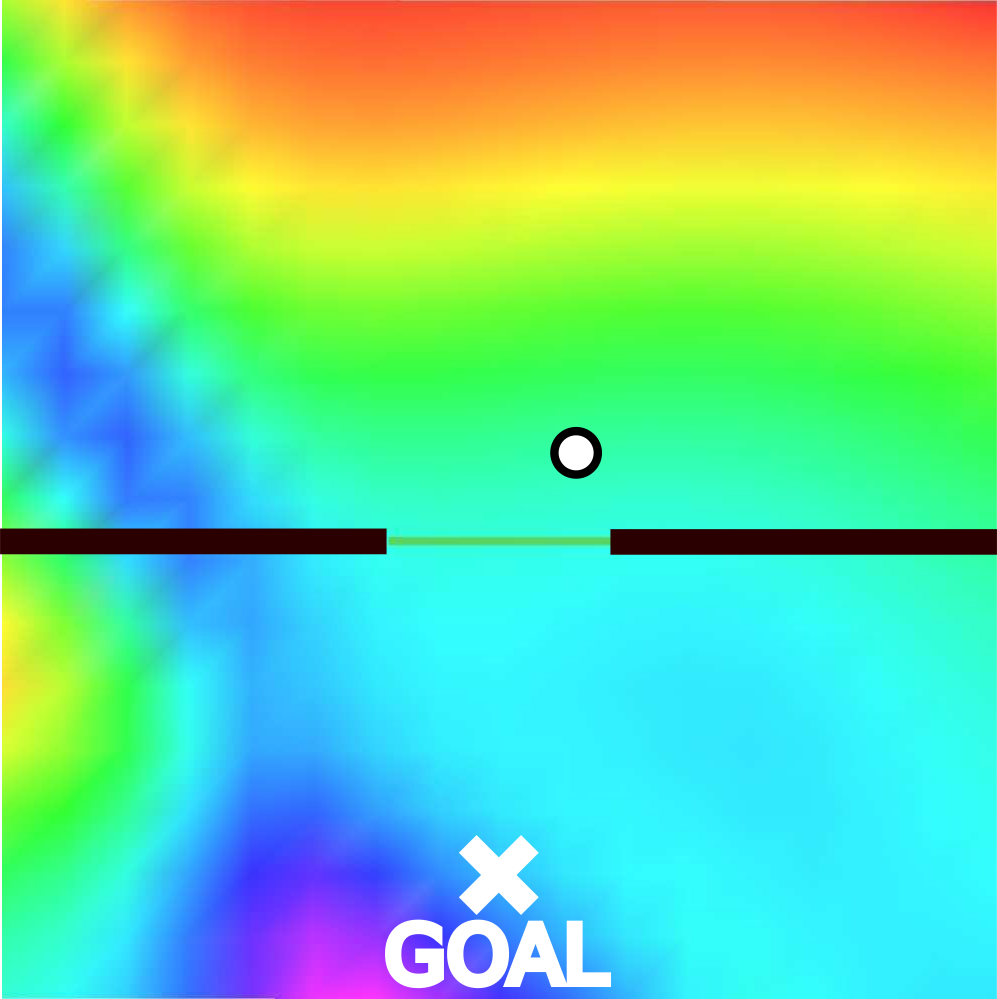}
    \includegraphics[width=.15\textwidth]{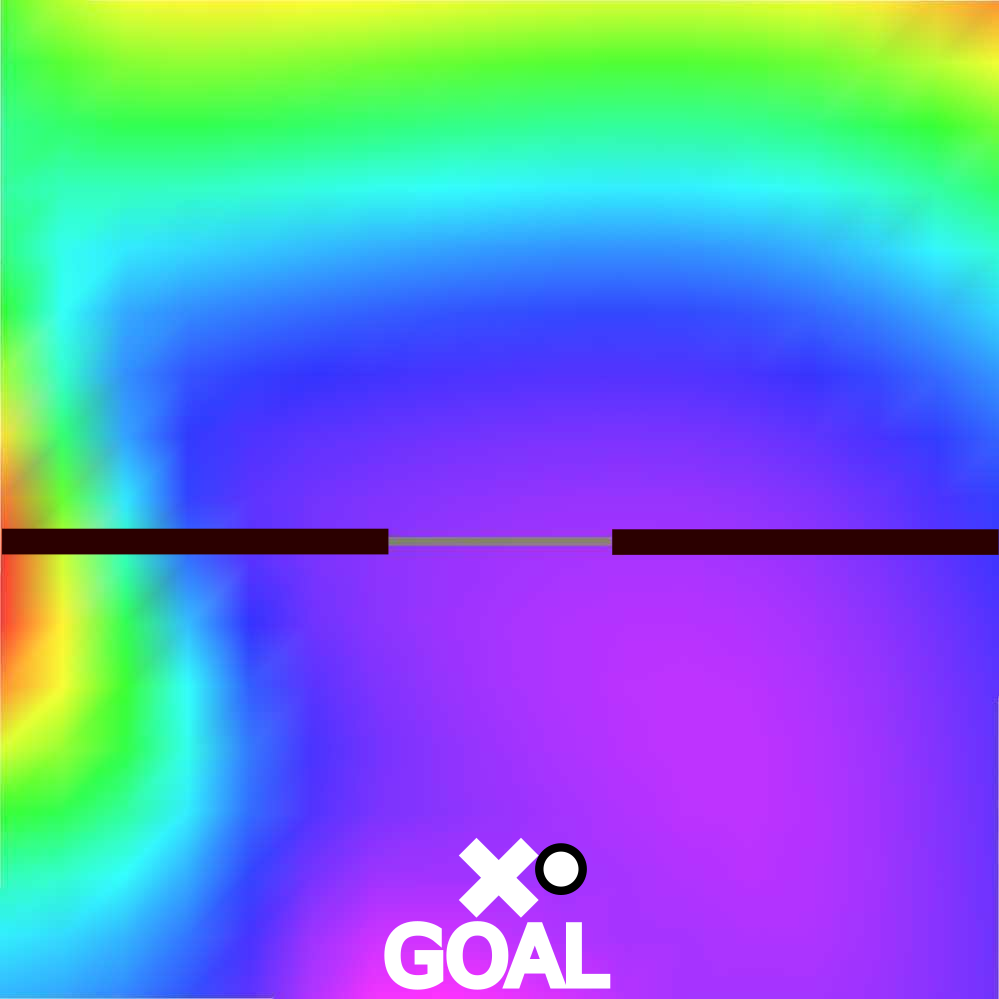}
    \hspace{.03\textwidth}
    \includegraphics[width=.15\textwidth]{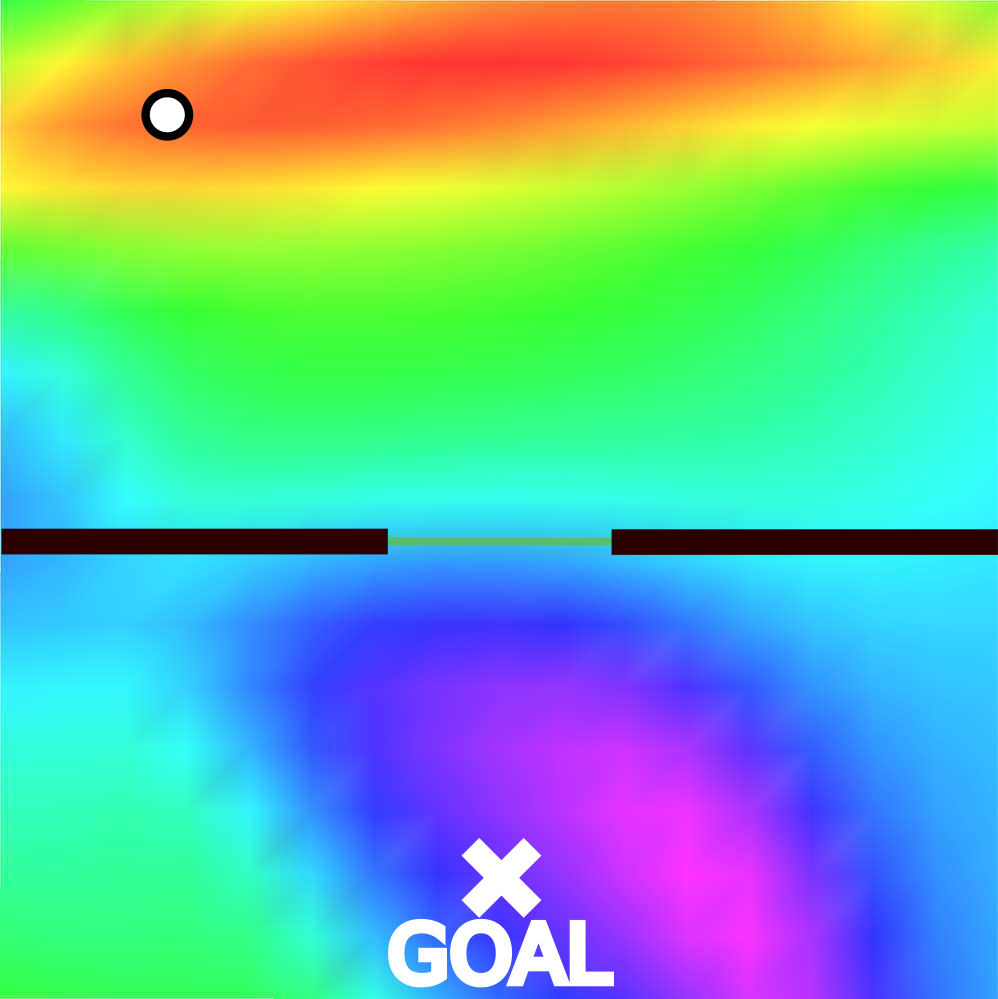}
    \includegraphics[width=.15\textwidth]{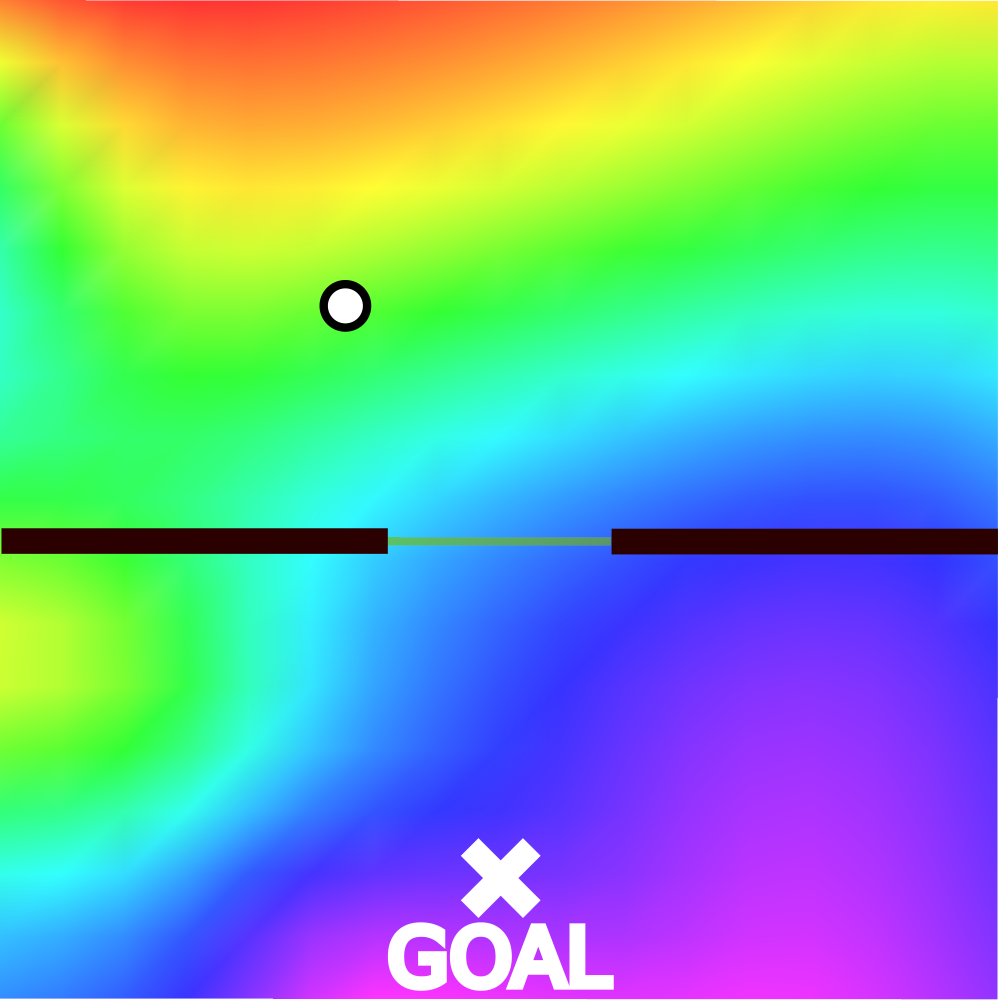}
    \includegraphics[width=.15\textwidth]{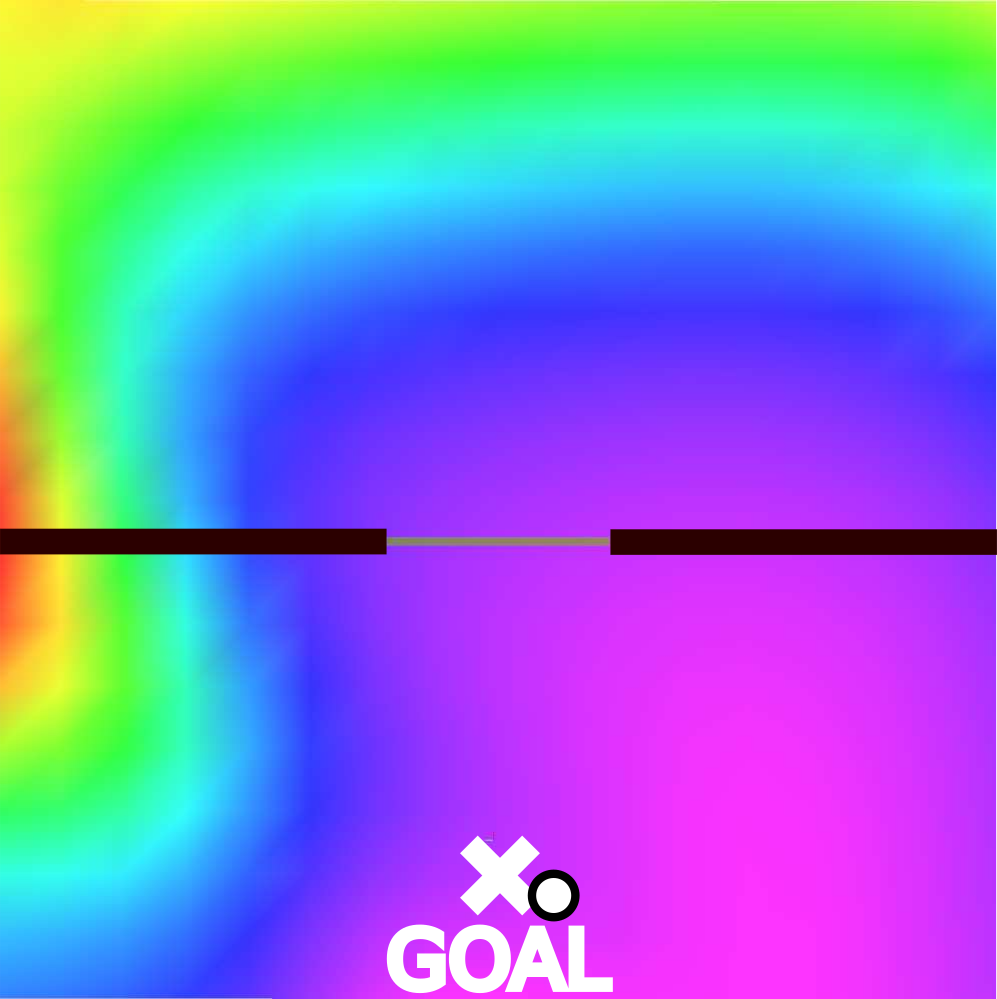}
    \caption{Evolution of the running cost \eqref{eq:intermediate_cost} along two trajectories (first three and last three images). The position of the Ballbot is denoted by the white dot and the target by the white cross. The heat-map represents the magnitude of the running cost at a given time for different x and y coordinates. Red corresponds to high cost and violet to low cost regions.}
    \vspace*{-4mm}
    \label{fig:result_heatmap}
\end{figure}

\paragraph{Influence of the Running Cost} As shown in Figure~\ref{fig:result_perf}, the running cost plays a vital role in solving the task. While the actors devoid of the running cost solve the reaching task well (it is a smooth and straightforward cost after all), they often collide with the walls. When the horizon is too long for them, the optimization is prone to overlooking the wall. This is because the value is only used at the end of the trajectory. Also, the exploration provided by the value function is not sufficient enough, and the Ballbot often collides with the wall. When using the running cost, the task is solved successfully upon convergence. The impact of that term is shown in Figure \ref{fig:result_heatmap}. The information encoded in the value function is able to produce regions of low cost and guide the solution to avoid the obstacles along the path. When the robot starts on the left side of the field, a region of low cost on the bottom right side is created, encouraging the robot to move forward and avoid the obstacle. The opposite happens when the robot starts on the right. Moreover, at the beginning of the trajectories, regions of high cost at the top of the field encourage the robot to move towards the goal.

\paragraph{Influence of the Horizon on the Convergence of the Value Function} 
As described in the previous section, the time horizon of the MPC plays a vital role. Indeed, with a longer time horizon, the actor can look further into the future during the trajectory optimization step. It accelerates the convergence of the value function and improves exploration. As can be seen in Figure \ref{fig:result_perf}, during training, the actor with the longer time horizon outperforms the one with a shorter one. The actor with a longer time horizon is able to see beyond the wall faster in order to reach the target. The actor with the shorter horizon tends to focus more on the wall and needs more time to get past it. On the other hand, increasing the time horizon results in a higher computational cost for the trajectory optimizer so that a trade-off has to be made.

%===============================================================================
\section{Related Work}
\label{sec:related_work}

% Planning and learning
The use of learning in conjunction with planning has been studied extensively in the past. Prominently, in the discrete domain, learning evaluation functions has been employed with tree search methods resulting in systems capable of planning over long time horizons \cite{silver2017AlphaGo, guez2018learning}.  
 %In \cite{zhang2015Quadrotor}, the behavior of an MPC is learned in a supervised learning fashion to control a quadrotor. %In \cite{pong2018TDM}, learns a time-varying and goal-conditioned value function that also acts as a surrogate for a model in an MPC setting. 
In inverse reinforcement learning/optimal control \cite{ross2010Dagger, abbeel2004IRL, dvijotham2010IOC}, the planner is taught to match the behavior of an expert by inferring a cost function. These methods, however, are bounded by the performance of the expert, which is assumed to showcase optimal behavior. In our problem setting, the cost function is inferred indirectly via the value function that is learned from the environment-issued rewards/costs. The performance is only bounded by the capacity to learn the optimal value function. 

The use of trajectory optimization with value function learning has been studied most recently in the Plan Online Learn Offline framework \cite{lowrey2018plan}. By using the value function as the terminal cost of their trajectory optimizer, they show improved performance of the policy beyond the optimizer's time horizon. Here, we further extend their idea to handle stochastic systems explicitly and formulate the optimizer's running cost such that it results in an importance sampler of the optimal value function. 

% Planning and exploration
When combining planning and learning, exploration plays a key role because it is difficult to cover the task-relevant state space in high-dimensional problems efficiently. To this end, methods such as path integral optimal control employ importance sample schemes \cite{kappen16adaptive}. In RL, methods such as Guided Policy Search \cite{levine2013GPS} use a planner to direct the policy learning stage and sample more efficiently from high reward regions. %I2A \cite{racaniere2017I2A}, planning is used to guide exploration during learning.

%===============================================================================
\section{Conclusion and Future Work}
\label{sec:discussion}
In this paper, we presented an off-policy actor-critic algorithm called DMPC that extends previous work on the combination of trajectory optimization and global value function learning. We first show that using a value function in the heuristic function leads to a temporally global optimal solution. Next, we show that using the running cost \eqref{eq:intermediate_cost} results not only in an importance sampling scheme that improves the convergence of the value function estimation but is also capable of taking system uncertainties into account.
In future work, we like to extend the value function to encode more information. For example, we could condition it on a local robot-centric map that would allow it to make decisions in dynamic environments to avoid obstacles.

\clearpage
% The acknowledgments are automatically included only in the final version of the paper.
\acknowledgments{This work was supported by NVIDIA, the Swiss National Science Foundation through the National Centre of Competence in Research Robotics (NCCR Robotics), the European Research Council (ERC) under the European Union’s Horizon 2020 research and innovation programme under grant agreement No. 852044.}

%===============================================================================

\bibliography{ms}

\clearpage
\appendix

%===============================================================================
%===============================================================================
%===============================================================================
\section{Proof of Theorems}
\label{app:proof}

%===============================================================================
\subsection{Girsanov Theorem}

For the stochastic process defined in Equation~\eqref{eq:stochastic_process}, it is possible to relate the state trajectory distribution $p$ of two policies $\pi_1$ and $\pi_2$ via the Girsanov theorem. Here, we briefly outline the steps of an informal derivation by discretizing the state trajectory with infinitesimal time $dt$. The distribution of $\vx(t+dt)$ conditioned on $\vx(t)$ is a Gaussian distribution with mean ${\vmu_{\pi}(t) = \vx(t) + \vf(t, \vx(t)) dt + \vg(t, \vx(t)) \vu_\pi(t) dt}$ and variance ${\vXi(t) = \vg(t,\vx) \vR^{-1} \, \vg(t,\vx)^\top}$. As a result, for the state trajectory distribution under policy $\pi_1$ we have
\begin{align*}
    p^{\pi_1}(\rho) 
    =& \lim_{dt \to 0} \prod^{T-dt}_{s=0} \cN(\vmu_{\pi_1}(s), \mathbf{\Xi}(s)dt) 
    \\
   \propto& 
    \lim_{dt \to 0} \prod^{T-dt}_{s=0} 
        \exp \left( -\frac 12
            \left\lVert \vx(s+dt) - \vmu_{\pi_1}(s) \right\rVert_{\vXi^{-1}}^2 dt^{-1}
        \right)
    \\
   \propto& 
    \lim_{dt \to 0} \prod^{T-dt}_{s=0} 
        \exp \left( -\frac 12
            \left\lVert (\vx(s+dt) - \vmu_{\pi_2}(s)) - (\vmu_{\pi_1}(s) - \vmu_{\pi_2}(s)) \right\rVert_{\vXi^{-1}}^2 dt^{-1}
        \right).
\end{align*}
After further simplifications,
\begin{align*}
    p^{\pi_1}(\rho)
    =& \lim_{dt \to 0} \prod^{T-dt}_{s=0} 
        \cN(\vmu_{\pi_2}(s), \vXi(s)dt) 
        \exp \left( 
            - \frac12 \left\lVert \vmu_{\pi_1}(s) - \vmu_{\pi_2}(s) \right\rVert_{\vXi^{-1}}^2 dt^{-1}
        \right)
        \\
        & ~~~~~~~~~~~~~~~~~~~~~~~
        \exp \Big( 
            (\vx(s+dt) - \vmu_{\pi_2}(s))^\top \vXi(s)^{-1} (\vmu_{\pi_1}(s) - \vmu_{\pi_2}(s))  dt^{-1} 
        \Big)
    \\
    =& 
        p^{\pi_2}(\tau) 
        \exp \Big( 
            - \frac12 \int_{0}^{T} \left\lVert \vu_1(t) - \vu_2(t) \right\rVert_{\vR}^2 dt
        \Big) \,
        \exp \Big( 
            \int_{0}^{T} (\vu_1(t) - \vu_2(t))^\top \vR \, d\cB(t)) 
        \Big).
\end{align*}
Thus we have
\begin{align}
    \frac{p^{\pi_1}(\rho)}{p^{\pi_2}(\tau) } 
    =& 
        \exp \Big( 
            - \frac12 \int_{0}^{T} \left\lVert \vu_1(t) - \vu_2(t) \right\rVert_{\vR}^2 dt
        \Big) \,
        \exp \Big( 
            \int_{0}^{T} (\vu_1(t) - \vu_2(t))^\top \vR \, d\cB(t)) 
        \Big)
        \label{eq:Girsanov_theorem}
\end{align}
Using the Girsanov theorem, it can be readily shown that the KL divergence between $p^{\pi_1}$ and $p^{\pi_2}$ takes the following form
\begin{align}
    \KL(p^{\pi_2} \Vert p^{\pi_1}) 
    =& 
    -\E_{p^{\pi_2}} \left[ \log \left( \frac{p^{\pi_1}(\rho)}{p^{\pi_2}(\rho)} \right) \right]
    =
    \E_{p^{\pi_2}} \left[ \int_{0}^{T}
            \frac12 \left\lVert \vu_1(t) - \vu_2(t) \right\rVert_{\vR}^2 dt \right]
    \label{eq:Girsanov_kl}
\end{align}

%===============================================================================
\subsection{Proof of Theorem~\ref{thm:kl_upper_bound}}
\label{app:proof:kl_upper_bound}

\begin{proof}

Based on the Girsanov theorem and Equation~\eqref{eq:Girsanov_kl}, the KL-divergence between the optimal and the current state trajectory distribution starting from an initial time $t$ and an initial state $\vx(t) = \vs$, can be written as 
\begin{align*}
    \KL(p^* \Vert p^{\pi}) 
    =& 
    \E_{p^*} \left[ \int_{t}^{T}
            \frac12 \left\lVert \vu^*(\tau) - \vu_{\pi}(\tau) \right\rVert_{\vR}^2 d\tau \right]
    =
    \E_{p^{\pi}} \left[ 
        \frac{p^*}{p^{\pi}}
        \int_{t}^{T} \frac 12 \lVert \vu^*(\tau) - \vu_\pi(\tau) \rVert_{\vR}^2 d\tau
    \right]
\end{align*}
Equation~\eqref{eq:optimal_distribution} defines the relationship in between the state trajectory probability distribution of the optimal policy, $\pi^*$, and the sampling policy, $\pi$ as  
\begin{align*}
    \KL(p^* \Vert p^\pi) 
    =&
    \E_{p^\pi} \left[ 
        \frac{e^{- \frac{1}{\lambda} C^{\pi}(t,\vs)}}{\Psi^* (t, \vs)} 
        \int_{t}^{T} \frac 12 \lVert \vu^*(\tau) - \vu_\pi(\tau) \rVert_{\vR}^2 d\tau 
    \right]
    \notag \\
    =&
    \KL(p^\pi \Vert p^*)
    +
    \E_{p^\pi} \left[ 
        \frac{e^{- \frac{1}{\lambda} C^{\pi}(t,\vs)} - \Psi^* (t, \vs)}{\Psi^* (t, \vs)} 
        \int_{t}^{T} \frac 12 \lVert \vu^*(\tau) - \vu_\pi(\tau) \rVert_{\vR}^2 d\tau
    \right]
    \notag \\
    =&
    \KL(p^\pi \Vert p^*)
    +
    \Delta
\end{align*}
where we have replaced the second term by $\Delta$. The second line naturally follows by using Equation~\eqref{eq:Girsanov_kl} for ${\KL(p^\pi \Vert p^*)}$. Now we further examine the term $\Delta$.
\begin{align*}
    \Delta^2
    \leq& 
    \E_{p^\pi} \left[ 
        \Big( \frac{e^{- \frac{1}{\lambda} C^{\pi}(t,\vs)} - \Psi^* (t, \vs)}{\Psi^* (t, \vs)} \Big)^2
    \right]
    \E_{p^\pi} \left[ 
        \Big( \int_{t}^{T} \frac 12 \lVert \vu^*(\tau) - \vu_\pi(\tau) \rVert_{\vR}^2 d\tau \Big)^2
    \right]
    \notag \\
    =&
    \Var \Big[ \frac{e^{- \frac{1}{\lambda} C^{\pi}(t,\vs)}}{\Psi^* (t, \vs)} \Big]
    \E_{p^\pi} \left[ 
        \Big( \int_{t}^{T} \frac 12 \lVert \vu^*(\tau) - \vu_\pi(\tau) \rVert_{\vR}^2 d\tau \Big)^2
    \right]
    \notag \\
    \leq&
    \Var \Big[ \frac{e^{- \frac{1}{\lambda} C^{\pi}(t,\vs)}}{\Psi^* (t, \vs)} \Big]
    \left( \cE \, \E_{p^\pi} \left[ 
        \int_{t}^{T} \frac 12 \lVert \vu^*(\tau) - \vu_\pi(\tau) \rVert_{\vR}^2 d\tau
        \right]
    \right)
\end{align*}
In the first line, we used the Cauchy-Schwarz inequality. Then, in the second line we have used the definition of the desirability function in Equation~\eqref{eq:desirability}.
Finally, in the third line, we used the Bhatia-Davis inequality which provides an upper bound on the variance of a bounded probability distribution. For the above we have
\begin{align*}
    0 \leq \int_{t}^{T} \frac 12 \lVert \vu^*(\tau) - \vu^\pi(\tau) \rVert_{\vR}^2 d\tau \leq \cE
\end{align*}
Thus, based on the Bhatia-Davis inequality we can write
\begin{align*}
    \E_{p^\pi} \left[ 
        \Big( \int_{t}^{T} \frac 12 \lVert \vu^*(\tau) - \vu_\pi(\tau) \rVert_{\vR}^2 d\tau \Big)^2
    \right] 
    - 
    \Big( \KL(p^\pi \Vert p^*) \Big)^2
    \leq&
    \Big( \cE - \KL(p^\pi \Vert p^*) \Big) \KL(p^\pi \Vert p^*)
    \\
    \E_{p^\pi} \left[ 
        \Big( \int_{t}^{T} \frac 12 \lVert \vu^*(\tau) - \vu_\pi(\tau) \rVert_{\vR}^2 d\tau \Big)^2
    \right] 
    \leq&
    \cE \KL(p^\pi \Vert p^*)
\end{align*}
Thus, we will have the following upper bound on $\Delta$
\begin{align*}
    \Delta
    \leq
    \lvert \Delta \rvert
    \leq
    \Big(
    \cE \,
    \KL(p^\pi \Vert p^*) \,
    \Var \Big[ \frac{e^{- \frac{1}{\lambda} C^{\pi}(t,\vs)}}{\Psi^* (t, \vs)} \Big] 
    \Big)^{\frac 12}
\end{align*}
\end{proof}

%===============================================================================
\subsection{Proof of Theorem~\ref{thm:mpc_upper_bound}}
\label{app:proof:mpc_upper_bound}

\begin{proof}

By using the relationship between the state trajectory probability distribution of the optimal policy, $\pi^*$, and the sampling policy, $\pi$, defined in Equation~\eqref{eq:optimal_distribution}, we get
\begin{align}
    \KL(p^{\pi} \Vert p^*) 
    =& 
    - \E_{p^{\pi}} \left[ 
        \log \left( 
        \frac{p^*}{p^{\pi}}
        \right) 
    \right]
    \notag \\
    =& 
    - \E_{p^\pi} \left[ 
        \log \left( 
        \frac{e^{- \frac{1}{\lambda} C^{\pi}(t,\vs)}}{\Psi^* (t, \vs)}
        \right)
    \right]
    \notag \\
    =&
    \frac{1}{\lambda} \E_{p^{\pi}} \left[
        C^{\pi}(t,\vs)
        + \lambda \log \Psi^* (t, \vs)
    \right]
    \notag \\
    =&
    \frac{1}{\lambda} \E_{p^{\pi}} \left[ C(t,\vs) \right]
    - 
    \frac{1}{\lambda} \E_{p^{*}} \left[ C(t,\vs) \right]
\end{align}
where we used Equation~\eqref{eq:optimal_value_solution} and then ${\lambda \log \Psi^* (t, \vs) = - V^{*}(t,\vs) = -\E_{p^{*}} \left[ C(t,\vs) \right]}$.

The rest of this proof follows similar to the result presented in~\cite{lowrey2018plan}. The difference is that our formulation is continuous-time while the formulation in~\cite{lowrey2018plan} is discrete-time. For the sake of brevity, we will use ${ l(\tau,\vx,\vu) := \gamma^{\tau-t} q(\tau,\vx) + \frac 12 \vu^\top \vR \, \vu }$ during this proof. 
\begin{align}
    \E_{p^{*}} \left[ C(t, \vs) \right]
    &= 
    \E_{p^{*}} \left[ 
        \gamma^{T-t} q_f(\vx(T)) + \int_{t}^{T} l(\tau,\vx,\vu) d\tau
    \right]
    \notag \\
    &= 
    \E_{p^{*}} \left[ 
        \gamma^{H} V^*(t_f, \vx_f) + \int_{t}^{t_f} l(\tau,\vx,\vu) d\tau
    \right]
\end{align}
where $t_f=t+H$ and $\vx_f:=\vx(t+H)$. We here truncate the horizon of optimization to the time horizon of MPC. To compensate for the truncated cost, we use the optimal value function as the termination cost.

When the actor uses an MPC strategy, it only has access to the approximated value function.
\begin{align}
    \E_{p^{\pi}} \left[ C(t, \vs) \right]
    &= 
    \E_{p^{\pi}} \left[ 
        \gamma^{H} \hat{V}(t_f, \vx_f) + \int_{t}^{t_f} l(\tau,\vx,\vu) d\tau
    \right]
\end{align}
We can then write the KL divergence as
\begin{align}
    \KL(p^{\pi} \Vert p^*) 
    =&
    \frac{1}{\lambda} 
    \E_{p^{\pi}} \!\left[ 
    \gamma^{H} \hat{V}(t_f, \vx_f) + \! \int_{t}^{t_f} \!\!\! l(\tau,\vx,\vu) d\tau
    \right]
    - \frac{1}{\lambda}
    \E_{p^{*}} \!\left[ 
    \gamma^{H} V^*(t_f, \vx_f) + \! \int_{t}^{t_f} \!\!\! l(\tau,\vx,\vu) d\tau
    \right]
    \notag
\end{align}
Adding and subtracting 
$
\frac{1}{\lambda} 
\E_{p^{\pi}} \left[ 
\gamma^{H} V^*(t_f, \vx_f) + \int_{t}^{t_f} l(\tau,\vx) d\tau
\right]
$ 
\begin{align}
    \KL(p^{\pi} \Vert p^*) 
    =&
    \frac{\gamma^{H}}{\lambda} \E_{p^{\pi}} \left[ 
    \hat{V}(t_f, \vx_f)
    - V^*(t_f, \vx_f)
    \right]
    +
    \frac{1}{\lambda} 
    \E_{p^{\pi}} \left[ 
    \gamma^{H} V^*(t_f, \vx_f) + \int_{t}^{t_f} l(\tau,\vx,\vu) d\tau
    \right]
    \notag \\
    \phantom{=}&
    - \frac{1}{\lambda}
    \E_{p^{*}} \left[ 
    \gamma^{H} V^*(t_f, \vx_f) + \int_{t}^{t_f} l(\tau,\vx,\vu) d\tau
    \right]
    \label{eq:kl_reverse_incomplete}
\end{align}
Using our assumption ${\max_{t,s} \vert \hat V(t,\vs) - V^*(t,\vs) \vert} = \cL$,
\begin{align*}
        \E_{p^{\pi}} \left[ 
        \gamma^{H} V^*(t_f, \vx_f) + \int_{t}^{t_f} l(\tau,\vx,\vu) d\tau
        \right]
    &\leq 
        \E_{p^{\pi}} \left[ 
        \gamma^{H} \hat{V}(t_f, \vx_f) + \int_{t}^{t_f} l(\tau,\vx,\vu) d\tau
        \right]
        + \gamma^{H} \cL
    \notag \\
        \E_{p^{*}} \left[ 
        \gamma^{H} V^*(t_f, \vx_f) + \int_{t}^{t_f} l(\tau,\vx,\vu) d\tau
        \right]
    &\geq 
        \E_{p^{*}} \left[ 
        \gamma^{H} \hat{V}(t_f, \vx_f) + \int_{t}^{t_f} l(\tau,\vx,\vu) d\tau
        \right]
        - \gamma^{H} \cL
        \notag \\
\end{align*}
using these bounds in Equation~\eqref{eq:kl_reverse_incomplete}, we get
\begin{align}
    \KL(p^{\pi} \Vert p^*) 
    \leq&
    \frac{\gamma^{H}}{\lambda} \E_{p^{\pi}} \left[ 
    \hat{V}(t_f, \vx_f)
    - V^*(t_f, \vx_f)
    \right]
    +
    \frac{1}{\lambda} 
    \E_{p^{\pi}} \left[ 
    \gamma^{H} \hat{V}(t_f, \vx_f) + \int_{t}^{t_f} l(\tau,\vx,\vu) d\tau
    \right]
    \notag \\
    \phantom{=}&
    - \frac{1}{\lambda}
    \E_{p^{*}} \left[ 
    \gamma^{H} \hat{V}(t_f, \vx_f) + \int_{t}^{t_f} l(\tau,\vx,\vu) d\tau
    \right]
    +
    2 \frac{\gamma^{H} \cL}{\lambda}
\end{align}
Since $p_{\pi}$ is generated to minimize ${ \E_{p^{\pi}} \left[  \gamma^{H} \hat{V}(t_f, \vx_f) + \int_{t}^{t_f} l(\tau,\vx,\vu) d\tau \right] }$, therefore 
\begin{align*}
    \E_{p^{\pi}} \left[  \gamma^{H} \hat{V}(t_f, \vx_f) + \int_{t}^{t_f} l(\tau,\vx,\vu) d\tau \right]
    \leq
    \E_{p^{*}} \left[  \gamma^{H} \hat{V}(t_f, \vx_f) + \int_{t}^{t_f} l(\tau,\vx,\vu) d\tau \right]
\end{align*}
We then have
\begin{align}
    \KL(p^{\pi} \Vert p^*) 
    \leq&
    \frac{\gamma^{H}}{\lambda} \E_{p^{\pi}} \left[ 
    \hat{V}(t_f, \vx_f)
    - V^*(t_f, \vx_f)
    \right]
    +
    2 \frac{\gamma^{H} \cL}{\lambda}
\notag \\
    \leq&
    \gamma^{H} \left[ 
    \KL(p^{\pi}_f \Vert p^*_f)
    \right]
    +
    2 \frac{\gamma^{H} \cL}{\lambda}
\end{align}
Recursively applying the KL bound to $\KL(p^{\pi}_f \Vert p^*_f)$, we get
\begin{align}
    \KL(p^{\pi} \Vert p^*)
    \leq&
    \frac{2 \cL \gamma^{H}}{\lambda} \left(
    1 + \gamma^{H} + \gamma^{2H} + \cdots
    \right)
    \notag \\
    \leq&
    \frac{2 \cL \gamma^{H}}{\lambda (1-\gamma^{H})} 
    \notag
\end{align}%
\end{proof}

%===============================================================================
\subsection{Proof of Proposition~\ref{thm:running_cost}}
\label{app:proof:running_cost}

First, we note that HJB equation for the problem defined in Equations \eqref{eq:stochastic_process}-\eqref{eq:path_cost} with $\vSigma \, \vR = \lambda \vI$ has the form
\begin{align}
\label{eq:stochastic_hjb}
    - \partial_t V^*(t,\vx) =
    & 
        l(t, \vx) + \partial_{\vx} V^*(t, \vx)^\top \vf(t, \vx) 
        - \frac{1}{2} \partial_{\vx} V^*(t, \vx)^\top \mathbf{\Xi}(t, \vx) \, \partial_{\vx} V^*(t, \vx)
    \notag \\
    &
        + \frac{\lambda}{2} \mathtt{Tr} [\partial_{\vx}^2 V^*(t, \vx) \, \mathbf{\Xi}(t, \vx)]
        ,
\end{align}
with $V^*(T,\vx) = \phi(\vx(T))$ and $\mathbf{\Xi}(t, \vx) = \vg(t, \vx) \vR^{-1} \, \vg(t, \vx)^\top = \lambda \, \vg(t, \vx) \vSigma \, \vg(t, \vx)^\top$. For sake of brevity, we define ${l(t, \vx):=\gamma^{t-t_0} q(t, \vx)}$ and ${\phi(\vx):=\gamma^{T-t_0} q_f(\vx)}$.

The optimal control policy can be derived as
\begin{align}
\label{eq:lmdp_policy}
    \pi^*(t, \vx) = -\vR^{-1} \vg(t, \vx)^\top \partial_{\vx} V^*(t, \vx).
\end{align}

\begin{lemma} \label{thm:deterministic_equivalence}

The optimal control policy of the stochastic problem in equations  \eqref{eq:stochastic_process}-\eqref{eq:path_cost} is the optimal solution to a deterministic problem with following cost functional
\begin{align}
\label{eq:deterministic_cost}
    C^{\pi}_d(t_0, \vs_0) 
    =& 
    \gamma^{T-t_0} q_f(\vs(T)) 
    + \int_{t_0}^T 
    \Big(
        \gamma^{t-t_0} q(t, \vs) + \frac{\lambda}{2} \Tr [\mathbf{\Xi}(t, \vs) \partial_s^2 V(t, \vs)]
        + \frac 12 \vu^\top \vR \, \vu   
    \Big) dt 
\end{align}
Where the state evolution is based on the mean of the system dynamics defined in \eqref{eq:stochastic_process}, i.e.,
\begin{align}
\label{eq:deterministic_process}
    & \dot \vs = \vf(t, \vs) + \vg(t,\vs) \vu, \quad \vs(t_0) = \vs_0
    \\
    & \vu = \pi(t, \vs).
\end{align}
\end{lemma}

\begin{proof}
The proof easily follows using the definition of the HJB equation for the stochastic and deterministic optimal control problems. 
\end{proof}

We finally provide the proof of Proposition~\ref{thm:running_cost}.

\begin{proof}
Start from the Girsanov theorem and Equation~\eqref{eq:Girsanov_kl}, we get
\begin{align}
    \KL(p^\pi \Vert p^*) 
    =& 
    \E_{p^\pi} \left[ 
        \int_{t}^{T} \frac 12 \lVert \vu_\pi(\tau) - \vu^*(\tau) \rVert_{\vR}^2 d\tau 
    \right]
    \notag \\
    =&
    \E_{p^\pi} \Bigg[ 
        \int_{t}^{T} \Big(
        \frac 12 \lVert \vu_\pi(\tau) \rVert_{\vR}^2 
        + \partial_{\vx} V^*(s,\vx)^\top \vg(\tau,\vx) \vu_\pi(\tau))
        \notag \\
        & ~~~~~~~~~~~~~~~~~~~
        + \frac 12 \partial_{\vx} V^*(\tau,\vx)^\top \mathbf{\Xi}(\tau, \vx) \, \partial_{\vx} V^*(\tau,\vx))
        \Big) d\tau 
    \Bigg] \label{eq:1}
     \\
    =&
    \E_{p^\pi} \Bigg[ 
        V^*(T,\vx(T)) 
        +
        \int_{t}^{T} \Big(
        \frac 12 \lVert \vu_\pi(\tau) \rVert_{\vR}^2 
        - \partial_t V^*(\tau,\vx)
        \notag \\
        & ~~~~~~~~~~~
        - \partial_{\vx} V^*(\tau,\vx)^\top \vf(\tau,\vx) 
        - \frac 12 \Tr \big[ \mathbf{\Xi}(\tau, \vx)) \partial_{\vx}^2 V^*(\tau,\vx)) \big]
        \notag \\
        & ~~~~~~~~~~~
        + \frac 12 \partial_{\vx} V^*(\tau,\vx)^\top \mathbf{\Xi}(\tau, \vx) \, \partial_{\vx} V^*(\tau,\vx)
        \Big) d\tau 
    \Bigg]
    - V^*(t,\vx) \label{eq:2}
     \\
    =&
    \E_{p^\pi} \Bigg[ 
        V^*(T,\vx(T)) 
        +
        \int_{t}^{T} \Big(
        L_e(\tau,\vx)
        + \frac 12 \lVert \vu_\pi(\tau) \rVert_{\vR}^2 
        \notag \\
        & ~~~~~~~~~~~
        - \frac 12 \Tr \big[ \mathbf{\Xi}(\tau, \vx) \partial_{\vx}^2 V^*(\tau,\vx)) \big]
        \Big) d\tau 
    \Bigg]
    - V^*(t,\vx) 
\end{align}
In Equation~\eqref{eq:1} we have replaced $\vu^*$ using Equation~\eqref{eq:lmdp_policy}; in Equation~\eqref{eq:2} we have used It\^o's Lemma for the process $\cV^*(t) = V^*(t, \vx)$
\begin{align}
\label{eq:ito}
    d \cV^*(t) 
    =& \partial_t V^*(t,\vx) dt
    + \partial_x V^*(t,\vx)^\top 
        \big( 
        \vf(t,\vx) + \vg(t,\vx) \vu_\pi(t, \vx) 
        \big) dt
    \notag \\
    & ~~ 
    + \frac{\lambda}{2} \Tr \big[ \mathbf{\Xi}(t, \vx) \partial_{\vx}^2 V^*(t,\vx) \big] dt
    + \partial_{\vx} V^*(t,\vx)^\top \vg(t,\vx) d\cB(t).
\end{align}
After reordering the terms in Equation~\eqref{eq:ito}, taking the time integral over the interval $[t, T]$, and taking the expectation with respect to $p^\pi$, we get
\begin{align}
    &\E_{p^\pi} \Bigg[
    \int_{t}^{T}
    \partial_{\vx} V^*(\tau,\vx)^\top \vg(\tau,\vx) \vu_\pi(\tau) d\tau
    \Bigg]
    = 
    \E_{p^\pi} \Bigg[ 
    V^*(T,\vx) 
    - \int_{t}^{T} \Big(
    \partial_t V^*(\tau,\vx)
    \notag \\
    & ~~~~~~~
    + \partial_{\vx} V^*(\tau,\vx)^\top \vf(t,\vx) 
    + \frac{\lambda}{2} \Tr \big[ \mathbf{\Xi}(\tau, \vx) \partial_{\vx}^2 V^*(\tau,\vx) \big] 
    \Big) d\tau
    \Bigg]
    - V^*(t,\vx) 
\end{align}
where the term involving $d\cB(t)$ cancels out. \\
As a result, the policy $\vpi$ which minimizes the reverse KL-divergence can be derived as 
\begin{align}
    \pi^*(t, \vs)
    =& 
    \arg\min_{\pi} 
    \E_{p^\pi} \Bigg\{ 
        V^*(T,\vx(T)) 
        +
        \int_{t}^{T} \Big(
        l_d(\tau,\vx)
        + \frac{1}{2}  \vu_\pi^\top \vR\,  \vu_\pi
        \notag \\
        & ~~~~~~~~~~~~~~~~~~~~~
        - \frac{\lambda}{2} \Tr \big[ \mathbf{\Xi}(\tau, \vx) \partial_{\vx}^2 V^*(\tau,\vx) \big]
        \Big) d\tau 
    \Bigg\}
\end{align}
This expectation is based on the probability distribution generated by the stochastic system in \eqref{eq:stochastic_process}. According to Lemma~\ref{thm:deterministic_equivalence}, we can transform this optimization problem to an equivalent deterministic problem in Equation~\eqref{eq:equivalent_deterministic}.%
\end{proof}

%===============================================================================
%===============================================================================
%===============================================================================
\section{Experimental Details}
\label{app:experiment_details}

\paragraph{Platform}
The Ballbot robot depicted in Figure \ref{fig:ballbot} is a 3D inverted pendulum capable of balancing on a ball with the help of three actuators. The mathematical formulation of system dynamics can be found in \cite{Minniti2019WholeBodyMF}.
Since it balances on a single ball, it has dynamic stability, is omnidirectional, and is capable of carrying out agile movements. Due to its inherent instability and highly nonlinear dynamics, this robot can also be used as a testing platform to validate general control algorithms, which may be applied to other types of mobile platforms. 

\begin{figure}      %% or \columnwidth
    \centering
    \includegraphics[width=0.3\linewidth]{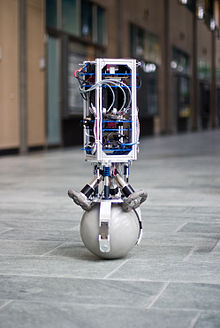}
    \caption{Image of the Ballbot used in the experiments. Ballbot is a torque-controlled, omnidirectional robot which balances on a ball through three omni-wheels.}
    \label{fig:ballbot}
\end{figure}

\paragraph{Experiments}
 An advantage of using an MPC strategy is that a nominal controller can be used to stabilize the system. In these experiments, we use an additional cost term in the trajectory optimizer that stabilizes the system in an upright position. As a result, the system will stand from the beginning resulting in much faster convergence to the desired behaviour. This is not a limitation of the pipeline, since the value function could also be trained to encode the upright stabilization.
 
In Figure \ref{fig:result_perf}, we tuned the hyper-parameters to achieve the best performance for each scenario. The performance at each learning iteration is computed by taking the average performance over 8 trajectories sampled from different starting positions.

%\subsection{Critic side}

%\subsection{Actor side}

%\subsection{Experiment Parameters}

\end{document}